\documentclass[letterpaper, 10 pt, conference]{ieeeconf}  

\IEEEoverridecommandlockouts                              
\usepackage{bmathshortcuts}

\newtheorem{theorem}{Theorem}

\usepackage{algorithmic}
\usepackage[ruled,vlined]{algorithm2e}
\usepackage{amsmath,amssymb,amsfonts}
\usepackage{bm}
\usepackage{booktabs}
\usepackage{cite}
\usepackage{color}
\usepackage{comment}
\usepackage{fancyhdr}
\usepackage{float}
\usepackage{gensymb}
\usepackage{graphicx}
\usepackage[hidelinks]{hyperref} 
\usepackage{mathrsfs}
\usepackage{soul}
\usepackage{textcomp}
\usepackage{times}
\usepackage[dvipsnames]{xcolor} 
\usepackage{url}

\newtheorem{remark}{Remark}

\newcommand{\planning}[1]{\noindent   {\colorbox{Mahogany}{\color{White}
 PLANNING:} \color{Mahogany} #1 \normalcolor}}

{\vspace{-\topsep}\begin{itemize}\itemsep1pt \parskip0pt \parsep1pt}
{\end{itemize}\vspace{-\topsep}}

\newenvironment{tightenumerate} 
{\vspace{-\topsep}\begin{enumerate}\itemsep1pt \parskip0pt \parsep1pt}
{\end{enumerate}\vspace{-\topsep}}

\begin{document}
\title{\LARGE \bf
Real-Time Trajectory Generation for Soft Robot Manipulators Using Differential Flatness}

 \author{Akua Dickson$^{1}$, Juan C. Pacheco Garcia$^{2}$, Ran Jing$^{2}$, Meredith L. Anderson$^{2}$, Andrew P. Sabelhaus$^{1,2}$
 \thanks{This work was in part supported by the U.S. National Science Foundation under Award No. 2340111, 2209783, and a Graduate Research Fellowship.}
 \thanks{$^1$A. Dickson and A.P. Sabelhaus are with the Division of Systems Engineering, Boston University, Boston MA, USA {\tt\small \{akuad, asabelha\}@bu.edu}. }
 \thanks{$^2$J.C. Pacheco Garcia, R. Jing, M.L. Anderson, and A.P. Sabelhaus are with the Department of Mechanical Engineering, Boston University, Boston MA, USA. {\tt\small \{jcp29, rjing, merland\}@bu.edu}.}
 }


\maketitle
\pagestyle{empty}  
\thispagestyle{empty} 



\begin{abstract}
Soft robots have the potential to interact with sensitive environments and perform complex tasks effectively.
However, motion plans and trajectories for soft manipulators are challenging to calculate due to their deformable nature and nonlinear dynamics. 
This article introduces a fast real-time trajectory generation approach for soft robot manipulators, which creates dynamically-feasible motions for arbitrary kinematically-feasible paths of the robot's end effector.
Our insight is that piecewise constant curvature (PCC) dynamics models of soft robots can be differentially flat, therefore control inputs can be calculated algebraically rather than through a nonlinear differential equation.
We prove this flatness under certain conditions, with the curvatures of the robot as the flat outputs. 
Our two-step trajectory generation approach uses an inverse kinematics procedure to calculate a motion plan of robot curvatures per end-effector position, then, our flatness diffeomorphism generates corresponding control inputs that respect velocity.
We validate our approach through simulations of our representative soft robot manipulator along three different trajectories, demonstrating a margin of 23x faster than real-time at a frequency of 100 Hz.
This approach could allow fast verifiable replanning of soft robots' motions in safety-critical physical environments, crucial for deployment in the real world.
\end{abstract}


\section{INTRODUCTION}
\label{Sec:introduction}
Soft robot manipulators have the potential to outperform traditional rigid-bodied robot manipulators in delicate tasks, navigating complex environments, and interacting safely with humans \cite{2015RusDesign, Laschi2016Soft, Sanan2011Physical}.
To perform these tasks, motion plans and trajectories are needed for the soft robot's poses and corresponding control inputs, all of which must obey the physics of the deformable body. 
However, significant challenges arise in both computational tractability and generalizability of the trajectory generation problem, due to these high dimensional and highly nonlinear dynamics \cite{Della2023Model}.
To date, motion planning for soft robots has been limited to slow-but-accurate methods or fast-but-approximated methods, with little work addressing dynamic feasibility in real time.

\begin{figure}[ht]
\centering
\includegraphics[width=0.95\linewidth]{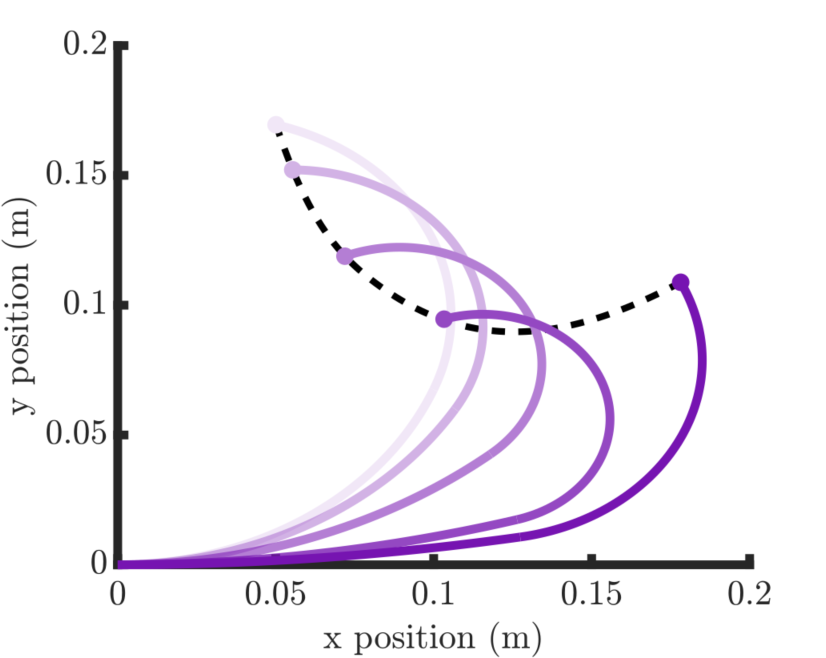}
\caption{The trajectory planning method in this manuscript generates dynamically feasible trajectories of a piecewise constant curvature soft robot manipulator (purple). A series of tip positions (black dotted lines) are tracked using a combination of fast inverse kinematics and differential flatness.} 
\label{fig:transparency}
\vspace{-0.3cm}
\end{figure}

This paper introduces an alternative approach for trajectory generation of soft robot manipulators.
We show that, under the common assumption of piecewise constant curvature (PCC) dynamics \cite{2020DellaModel-Based}, a soft robot manipulator is \textit{differentially flat}.
We then propose a motion planning problem combining inverse kinematics with control input calculations from differential flatness to generate arbitrary trajectories of the robot.
Our simulation validation shows that these motion plans are feasible under the robot's dynamics, and can be generated much faster than real-time (a margin of 23x faster at a frequency of 100 Hz).
This result could enable complex, verifiably-feasible motions of soft manipulators in time-critical environments.




\subsection{Kinematic Trajectory Generation For Soft Robots}

A number of techniques exist for trajectory generation and motion planning of soft robots, yet each suffer from some combination of limited assumptions about the robot, reliance on data, or computational challenges. 
Kinematically or statically, given some assumptions \cite{webster2010design,Camarillo2009}, an inverse kineto-statics problem can be solved directly \cite{Wang2021Inverse}, or via convex optimization problems \cite{Keyvanara2023AGeometric, Lai2022Constrained,Bern2019Trajectory}, or via sample-based planning methods \cite{luo_efficient_2024,khan_control_2020}.
However, kinematic and static approaches to motion planning do not consider the dynamic feasibility of \textit{trajectories} of these poses, therefore may introduce inaccuracy or constraint violation during more general soft robot manipulation tasks \cite{renda_screw-based_2017,renda_discrete_2018,boyer_dynamics_2020}




Similar concerns exist for discretized models which rely on static equilibrium, such as finite element methods (FEM/FEA).
Statics-based FEM can plan for complex deformations of a soft robot in 3D, in real-time with careful numerical choices \cite{duriez_control_2013,bern_soft_2021} or dimensionality reduction \cite{Ménager2023Direct, Largilliere2015Real-time, Coevoet2019Soft} or data-driven machine learning \cite{Terrile2023Use}.
Yet these real-time methods also ignore the dynamics of the soft robot.

\subsection{Dynamic Trajectories and Feedback Control for Soft Robot Manipulators}

Motion plans and control trajectories can also be generated with the soft robots' dynamics, but experience challenges in computation or approximations.
Two common dynamics models for soft robots, the strain-parameterized elastic rod and the Cosserat rod, both require solving partial differential equations.
Due to the subsequent computational load, most control or planning under these dynamics models occurs in simulation \cite{doroudchi_configuration_2021,renda_dynamics_2024,chang_energy_2020,caasenbrood_energy-shaping_2022}, released from the constraint of real-time computation.
These approaches may more accurately be called \textit{trajectory generation} rather than feedback control \cite{rao_trajectory_2014}, as they are not tested for online calculation.
Hardware implementation of Cosserat-rod based feedback control has been limited to simple robots only \cite{doroudchi_implementation_2023}.
In contrast, our approach is capable of calculating entire motion plans in microseconds, many times faster than when solving the Cosserat rod PDE \cite{till_real-time_2019}.



Alternatively, dynamics that map to the manipulator equation either through the Piecewise Constant Curvature (PCC) approximation \cite{della2020imporved,Sanders2023Dynamically} or a ``pseudo-rigid'' discretization into many small rigid links \cite{Wertz2022Trajectory} can trade some loss of accuracy for the benefits of ordinary differential equations (ODEs).
Both approaches have been validated as sufficiently accurate in hardware prototypes \cite{Della2023Model}, given certain conditions -- for example, the manipulator must be inextensible, with no transverse shear. 
These dynamics computations can be fast enough for one-step-ahead feedback even with advanced controllers \cite{Hyatt2020Model,patterson_robust_2022,patterson_safe_2024}.
Our contribution starts with the PCC model and provides a dramatic increase in computational speed, generating long-horizon trajectories much faster than real-time.

Finally, much of the recent success in real-time dynamic soft robot control has come through machine learning \cite{chen_data-driven_2024} or similar dimensionality-reduction techniques.
Finite element methods with dynamics can be made real-time via model-order reduction \cite{Wu2021FEM-based,tonkens_soft_2021,alora_data-driven_2023,Ménager2024Condensed}, yet are inherently an approximation with less physical insight or confidence than a first-principles model.
Data-driven techniques such as the Koopman Operator \cite{bruder_data-driven_2020,haggerty_control_2023,wang_improved_2023} or different varieties of neural networks \cite{thuruthel_model-based_2019,sun_physics-informed_2022} can be carefully used in real-time, but rely on quality of data and suffer in regions of the state space where no training data exists.
Our proposed method provides a ready physical interpretation (PCC), does not rely on data, and has a significant speedup in computation time.





    
\subsection{Differential Flatness in Robotics}

The primary insight for this manuscript is that the PCC dynamics of a soft robot manipulator can be \textit{differentially flat}.
Differential flatness is a property of some dynamical systems \cite{sira2004differentially} that allows their states and inputs to be written as an algebraic function of others, called the \textit{flat outputs}, and the flat outputs' derivatives.
Feasible trajectories can then be computed purely kinematically \cite{van_nieuwstadt_real_1997}, no ODE required.
The resulting trajectory generation problem can be very fast-to-compute, and has revolutionized the control of high-frequency feedback systems such as quadrotors \cite{Mellinger2011Minimum,  poultney2018Robust}, aerial manipulators \cite{ yuksel2016differential}, and robot arms \cite{Tonan2024Motion, Franch2010Differential, rouchon1993flatness}.
Whether used online for control or offline for trajectory generation, differential flatness is one of the fastest methods for calculating feasible inputs in nonlinear systems \cite{sun_comparative_2022}.




 
\subsection{Paper Contributions}

This article considers: if rigid robot manipulators are differentially flat, and if the equations of motion of a PCC soft robot can be written as a variation of the rigid manipulator dynamics, then are soft robot manipulators also differentially flat?
We show the answer is \textit{yes}.
In total, we contribute:


\begin{itemize}
\item A proof-by-construction that piecewise constant curvature soft robot dynamics can be \textit{differentially flat},
\item A real-time trajectory generation approach using this flatness property, and
\item A demonstration that our approach plans soft robot manipulator trajectories 23x faster than real-time.
\end{itemize} 

Unlike kinematic or static methods, our trajectories are dynamically feasible by construction.
And unlike model order reduction, machine learning, or the Koopman Operator, our approach is exact: flatness gives a one-to-one mapping that is valid whenever the PCC robot dynamics are valid.

\section{PIECEWISE CONSTANT CURVATURE SOFT ROBOT DYNAMICS}
\label{sec2:piecewise}

Our approach relies on the kinematics and dynamics of the piecewise constant curvature assumption in a soft robot limb.
We first review that model from \cite{2020DellaModel-Based,Della2023Model} before progressing to the flatness proof and planning approach.

\begin{figure}[tb]
\centering
\includegraphics[width=1.05\linewidth]{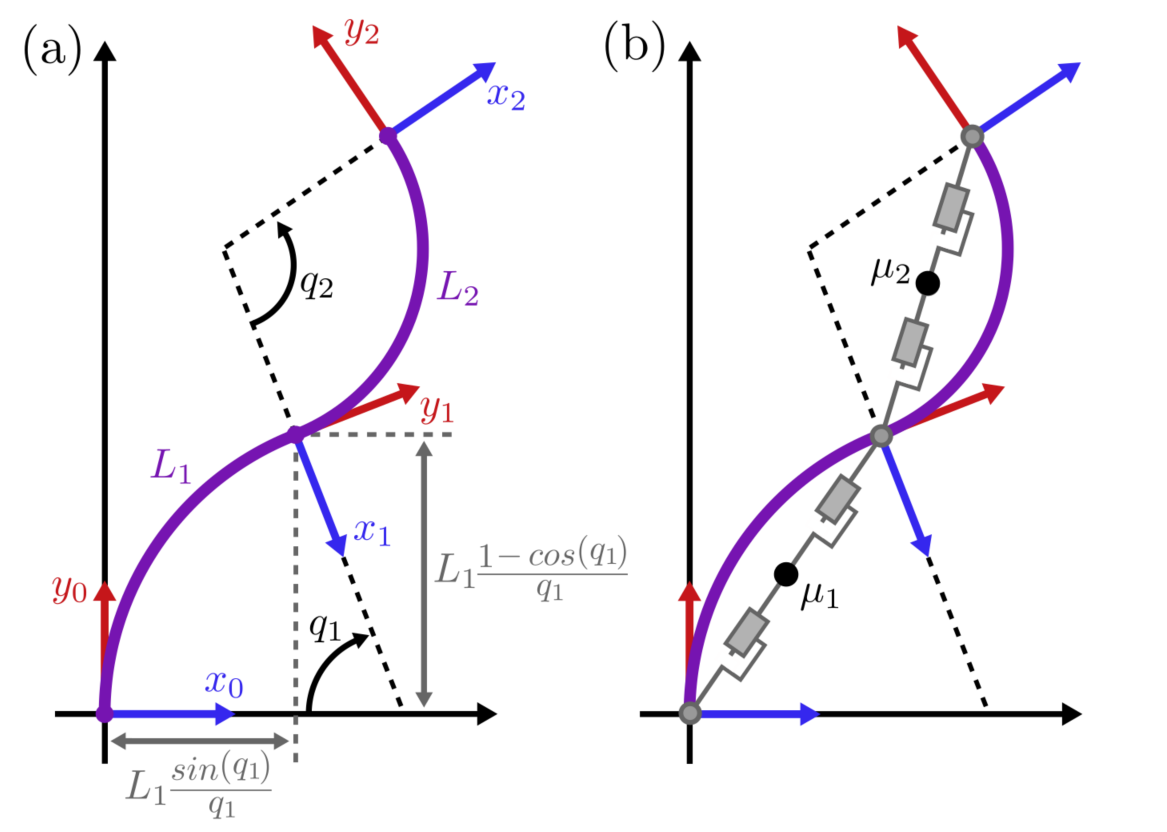}
\caption{(a) A piecewise constant curvature soft robot manipulator has kinematics defined by its subtended angles $q_i$ and lengths $L_i$. (b) The center of mass dynamics of a PCC robot can be matched to a constrained C.o.M., $\mu$, of rigid manipulator per \cite{della2020imporved}. Here, we use the RPPR representation.} 
\label{fig:rigid/softsegments}
\vspace{-0.3cm}
\end{figure}



\subsection{Kinematics Model}

The assumption that segments of a soft robot arm are constant-radius arcs (the PCC model, Fig. (\ref{fig:rigid/softsegments})) has seen much success \cite{webster2010design}.
Under this assumption, the robot's pose becomes a function of a finite number of states, amenable to planning.


We consider a PCC soft robot which comprises $n$ constant curvature segments, each with a reference frame at base and tip. 
The robot's kinematics are consequently defined in closed-form by a series of $n$ transformation matrices $T_0^1, \cdots, T_{n-1}^n$ which maps each reference frame to a successive one. 
This work considers a planar robot, where $T_{i-1}^i$ is defined using the curvature $q_i$ and length $L_i$ of each segment (Fig. \ref{fig:rigid/softsegments}(a)) as

\begin{equation}
\label{eq:forwardkinematics}
T_{i-1}^i = \begin{bmatrix}
\cos(q_i) & -\sin(q_i) & L_i \frac{\sin(q_i)}{q_i}\\
\sin(q_i) & \cos(q_i) & L_i \frac{1-\cos(q_i)}{q_i}\\ 0&0&1
\end{bmatrix}
\end{equation}


\subsection{Dynamics Model}

The work of Della Santina et al. proposes that the motion of one PCC arc can be modeled via a series of rigid rotational and prismatic joints, constrained such that their combined center of mass has the same dynamics as the PCC arc's center of mass \cite{2020DellaModel-Based}.
To derive this model, let the rigid robot configuration be $\xi \in \mathbb{R}^{nh}$, where $h$ is the number of rigid joints per each CC segment. 
The PCC configuration states $q \in \mathbb{R}^{n}$ are derived by constraining the rigid robot configuration through a continuously differentiable mapping:

\begin{equation}
    m: \mathbb{R}^{n} \longmapsto \mathbb{R}^{nh}
\end{equation}

The mapping $\xi = m(q)$ must ensure that kinematics and centers of mass of each CC soft segment via eqn. (\ref{eq:forwardkinematics}) aligns with the equivalent points of the rigid joints. 
For this, we adopt the RPPR configuration from \cite{2020DellaModel-Based} where a single CC segment is modeled as an RPPR manipulator (Fig. \ref{fig:rigid/softsegments}(b)),

\begin{equation}
m(q) = \begin{bmatrix}
m_1(q_1)^T&m_2(q_2)^T&\cdots&m_n(q_n)^T 
\end{bmatrix}^T   
\end{equation}

\noindent where each of the $m_i(q_i)$ are:

\begin{equation}
\label{eq:m_i(q)}
m_i(q_i) = \begin{bmatrix}
\frac{q_i}{2}\\L_i\frac{sin(\frac{q_i}{2})}{q_i}\\L_i\frac{sin(\frac{q_i}{2})}{q_i}\\ \frac{q_i}{2}    
\end{bmatrix}.
\end{equation}

\noindent We remark that, as observed in \cite{2020DellaModel-Based}, the RPPR $m(q)$ can be substituted for an RPRPR for an improvement (future work).

We start with the usual dynamics of the rigid robot:
\begin{equation}
\label{xi_dynamics}
B_\xi(\xi)\ddot{\xi} + C_\xi(\xi,\dot{\xi})\dot{\xi} + G_\xi(\xi) = J_\xi(\xi)^\top f_{ext}
\end{equation}
where $\xi$, $\dot{\xi}$ and $\ddot{\xi}$ $ \in \mathbb{R}^{nh}$ is the robot pose (joint angles and extensions) and its derivatives, $B_\xi(\xi) \in \mathbb{R}^{nh\times nh}$ is the rigid robot's inertia matrix, $C_\xi(\xi,\dot{\xi}) \in \mathbb{R}^{nh}$ is the Coriolis matrix, and $G_\xi(\xi) \in \mathbb{R}^{nh}$ is the gravity vector.
The generalized applied forces $f_{ext}$ will contain spring and damping terms as well as a control input \cite{2020DellaModel-Based}.




Applying $\xi = m(q)$ can be done per the chain rule.
Let the Jacobian of $m$ be $J_m(q):\mathbb{R}^{n} \mapsto \mathbb{R}^{nh \times n}$. 
Then the rigid robot configuration and its derivatives $\xi$, $\dot{\xi}$ and $\ddot{\xi}$, in terms of the CC soft segment states and derivatives, $q$, $\dot{q}$ and $\ddot{q}$, are

\begin{subequations}
\label{eq:xi=m(q)}
    \begin{align}
        \xi &= m(q)\\
        \dot{\xi} &= J_m(q)\dot{q} \\
        \ddot{\xi} &= \dot{J}_m(q,\dot{q})\dot{q} + J_m(q)\ddot{q}
    \end{align}
\end{subequations}

\noindent The Jacobian $J_m(q)$ is obtained by calculating $\frac{\partial m}{\partial q}$.
For the RPPR $m_i(q)$ of \eqref{eq:m_i(q)},

\begin{equation}
J_{m,i} = \begin{bmatrix}
    \frac{1}{2}&L_{m,i}&L_{m,i}&\frac{1}{2}
\end{bmatrix}^T
\end{equation}
where $L_{m,i} = L_i \frac{q_icos(\frac{q_i}{2})-2sin(\frac{q_i}{2})}{2q_i^2}$.

The dynamics of the PCC soft robot are then a simple substitution of \eqref{eq:xi=m(q)} into \eqref{xi_dynamics}:


\begin{equation}
\begin{aligned}
& B_\xi(m(q))(\dot{J}_m(q,\dot{q})\dot{q} + J_m(q)\ddot{q}) \\  & + C_\xi(m(q), J_m(q)\dot{q}) J_m(q)\dot{q} \\  & + G_\xi(m(q)) = J_\xi^\top(m(q))f_{ext},
\end{aligned}
\end{equation}

\noindent written more concisely in $\mathbb{R}^n$ by left-multiplication with $J_m^\top$,

\begin{equation}
B(q)\ddot{q} + C(q,\dot{q})\dot{q} + G(q) = J^\top(q) f_{ext}
\end{equation}

\noindent where

\begin{equation}\label{eqn:BCGJ_q}
\begin{aligned}
B(q) = & J_m^T(q)B_\xi(m(q))J_m(q) \\
C(q,\dot{q}) = & J_m^T(q)B_\xi(m(q))\dot{J}_m(q,\dot{q}) +J_m^T(q) \\
&C_\xi(m(q),J_m(q)\dot{q})J_m(q) \\
G(q) = & J_m^T(q)G_\xi(m(q))\\
J^T(q) = & J_m^T(q)J_\xi^T(m(q))
\end{aligned}
\end{equation}


\noindent As in \cite{2020DellaModel-Based}, we factor out $q$-space stiffness and damping from the applied force, with matrices $K, D \in \mathbb{R}^{n \times n} \succ 0$, and remove gravity since our robot is planar.
In addition, we consider an input generalized force $\tau$ within $f_{ext}$, with no external contact, obtaining a similar result from \cite{2020DellaModel-Based},


\begin{equation}
\label{eq:q_dynamics}
B(q)\ddot{q} + C(q,\dot{q})\dot{q}+Kq +D\dot{q} = J^\top(q) \tau
\end{equation}

\subsection{Input Model}

The physical interpretation and realizability of $\tau \in \mathbb{R}^n$ in eqn. (\ref{eq:q_dynamics}) is not obvious.
The approach from \cite{2020DellaModel-Based} treats $\tau_i$ for the $i$-th CC segment as a pair of internal torques at the $i$-th and $(i+1)$-th reference frames.
Here, it is assumed that the robot is fully actuated: there is one bidirectional input (e.g., antagonistic pneumatic chambers) per CC segment.
More importantly, they show that $\tau_i$ can be calibrated by a least-squares fit against a linear displacement of a pneumatic piston, $u_i$, as $\tau_i = \lambda_i u_i$, and that this curve fit accurately replaces $J^\top = J_m^\top J_\xi^\top$.
Combining these results implies $J^\top(q)\tau = J_\lambda^\top u$, where $J_\lambda = \text{diag}(\lambda_1, \hdots, \lambda_n)$ is a diagonal matrix whose entries are fit via hardware experiments.

Our envisioned hardware platform (see Sec. \ref{sec:discusion}) uses a similar pressure-controlled pneumatic system per CC segment, fully actuated, so we hypothesize that it is possible to calibrate $\lambda_i$ in future work.
This manuscript's simulations use $J_\lambda = I$.



\section{SOFT ROBOTS CAN BE DIFFERENTIALLY FLAT}
\label{sec3:soft}

In this section, we show a main result of this article: the dynamics of a PCC soft robot can be differentially flat, so the states and the input can be written as algebraic functions of the carefully chosen flat outputs and their derivatives.

\subsection{Differential Flatness}
\label{sec:differentialflatness}
A system of states $x\in \mathbb{R}^n$ and inputs $u \in \mathbb{R}^m$ is differentially flat if there exists a set of appropriately-smooth flat outputs $y = h(x,u,\dot{u},\ddot{u},...,u^{(a)}) \in \mathbb{R}^m$ such that:

\begin{align}
    x &= f(y,\dot{y},\ddot{y},...,y^{(b)})  \label{eq:f} \\
    u &= g(y,\dot{y},\ddot{y},...,y^{(c)}) \label{eq:g}
\end{align}

Notice that, if a system is differentially flat and $f$ and $g$ are known, one can specify an arbitrary trajectory of $y(t)$ then immediately calculate the corresponding states $x(t)$ and inputs $u(t)$ to satisfy that trajectory, via eqns. (\ref{eq:f})-(\ref{eq:g}).
The resulting trajectory is dynamically feasible by construction.

\subsection{Flat Output Selection}
\label{section:flatoutputselection}

Choosing a candidate set of flat outputs, $y$, is not trivial.
For our dynamics of eqn. (\ref{eq:q_dynamics}) with states $x = [q^\top \; \dot q^\top]^\top$ and inputs $u = \tau$, there is no general procedure to show flatness.

However, we observe that (some) rigid robot manipulators are differentially flat \cite{sira2004differentially} with the choice of joint angles as the flat outputs, i.e., $y = \xi$ from eqn. (\ref{xi_dynamics}).
We therefore hypothesize that, for the same reasons as the PCC dynamics can be written as a constrained rigid manipulator, our flat outputs could also be the constrained pose: i.e., $y = q$, with $q \in \mathbb{R}^{n}$ the angles of each constant curvature segment.

We consider eqn. (\ref{eq:q_dynamics}) and prove the following:




\subsection{Differential Flatness for PCC Soft Robot Dynamics}

\begin{theorem}\label{thm:df}
    For a soft robot manipulator, if:
    \begin{tightenumerate}
        \item The dynamics of the piecewise constant curvature rigid approximation in (\ref{xi_dynamics})-(\ref{eq:q_dynamics}) hold,
        \item there is one (bidirectional) actuator per PCC segment,
    \end{tightenumerate} 
    Then it is differentially flat with the flat outputs as $y=q$. If in addition, (3) a smooth trajectory of flat outputs $y(t)=q(t)$ does not include poses with singularities, then the corresponding state-input trajectory $\{x(t),u(t)\} = \{f(y^{(\cdot)}(t)), g(y^{(\cdot)}(t))\}$ is smooth and dynamically feasible.
\end{theorem}

\begin{proof}
    Choose $y(t)=q(t)$, then $\dot y = \dot q$ and $\ddot y = \ddot q$.
    We immediately have $x = f(y, \dot y) = [y^\top \; \dot y^\top]^\top$, satisfying eqn. (\ref{eq:f}).
    Next, by assumption (1), rearranging eqn. (\ref{eq:q_dynamics}) with $J^\top(q)\tau = J_\lambda^\top u$ gives
    
    \begin{equation}\label{eqn:u_diffflat}
        u = (J_\lambda^\top)^{-1}(B(y)\ddot{y} + C(y,\dot{y})\dot{y}+ Ky + D\dot{y}).
    \end{equation}

    \noindent By assumption (2), $J_\lambda$ is invertible and $u\in \mathbb{R}^n$ is the same dimension as $y$, so eqn. (\ref{eqn:u_diffflat}) serves as $g(\cdot)$ in eqn. (\ref{eq:g}), completing the flatness proof.
    Finally, applying assumption (3), all Jacobians in eqn. (\ref{eqn:BCGJ_q}) are singularity free for all values $q(t)$, true also for eqn. (\ref{eqn:u_diffflat}), therefore $u\in \mathcal{C}^\infty$.
\end{proof}
    


\begin{remark}
The result of Theorem \ref{thm:df} is fairly obvious, and is essentially the same proof as feedback linearizability \cite{sira2004differentially} or even computed torque control, just from a different perspective.
For other related manipulator-like models of soft robots, e.g. \cite{della2020imporved}, it is likely true that these are differentially flat as well, subject to the condition of full actuation.
\end{remark}

\begin{remark}
Soft robots in general are underactuated; our result arises from the one-actuator-per-segment assumption that holds for wide classes of soft manipulators \cite{Hyatt2020Model}.
However, for soft robots when this is not true (e.g. \cite{Wertz2022Trajectory,pacheco_garcia_comparison_2023}), many flatness proofs from underactuated rigid robots \cite{agrawal_differentially_2008,Franch2010Differential} could be adapted.
\end{remark}

\section{REAL-TIME PCC SOFT ROBOT TRAJECTORY GENERATION}
\label{sec4:real}
Our problem statement seeks a trajectory of states and inputs $\{x_0,u_0,\hdots,x_t,u_T\}$ that correspond to a task-space trajectory.
Differential flatness gives us a $u_t$ corresponding to $x_t$ via eqn. (\ref{eqn:u_diffflat}), but we must additionally map a sequence of end-effector positions $r_t = \begin{bmatrix} r_{x_t} & r_{y_t} \end{bmatrix}^\top \in \mathbb{R}^{2}$ to $q_t$ and $\dot q_t$.
We propose a combined trajectory generation algorithm that calculates $q_t$ given $r_t$ via inverse kinematics, using a finite difference approximation for velocities $\dot q_t$ and accelerations $\ddot q_t$.
This simple proof-of-concept scheme demonstrates one way to combine kinematic planning with dynamic trajectory generation; higher performance is possible in future work.





\begin{figure*}
\centering
\includegraphics[width=1\textwidth]{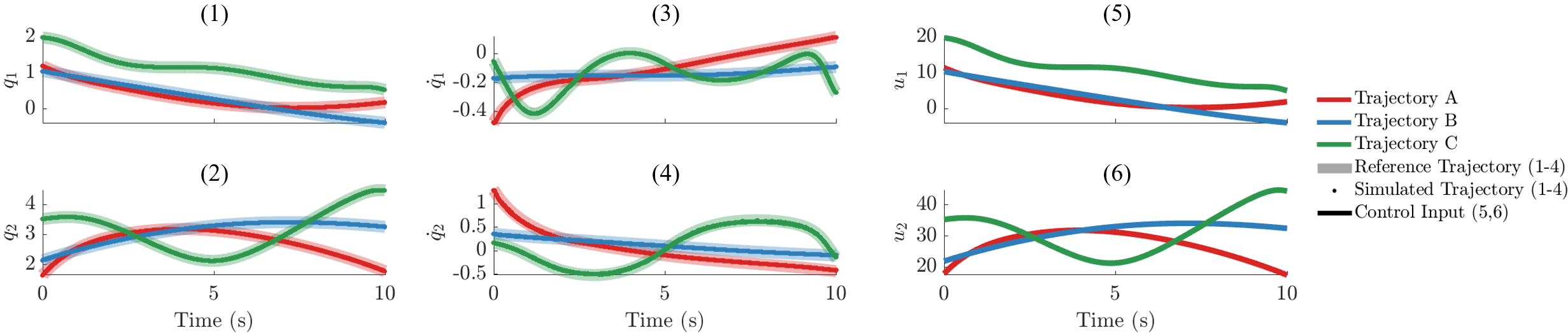}
\caption{With a sufficiently-small $\Delta t$, our method calculates open-loop trajectories that track the three reference trajectories A, B and C with low error.} \label{fig:jointangleplots}
\end{figure*}

\subsection{Inverse Kinematics}
\label{sec:InverseKinematics}

Any number of advanced inverse kinematics planners exist for soft robots \cite{Camarillo2009}, and could be used here.
However, our envisioned proof-of-concept robot contains only $n=2$ planar segments, so we can adapt a simple approach of numerically solving a system of nonlinear equations \cite{Wang2021Inverse}.

Taking two PCC segments' kinematics via eqn. (\ref{eq:forwardkinematics}), we obtain the transformation matrix $T_{12} = [R(q), \; r(q); \; 0, \; 1]$, where the tip position $r(q) \in \mathbb{R}^2$ is




\begin{align}
    r_x = & L_1 \frac{\sin(q_1)}{q_1} + \frac{L_2}{q_2} \left( \sin(q_2)\cos(q_1) \right. \notag \\
    & \quad \quad \left. - \sin(q_1)(1-\cos(q_2)) \right) \label{eq:x}\\
    r_y = & L_1 \frac{1-\cos(q_1)}{q_1} + \frac{L_2}{q_2} \left( \sin(q_1)\sin(q_2) \right. \notag \\
    & \quad \quad \left. + \cos(q_1)(1-\cos(q_2)) \right) \label{eq:y}
\end{align}


\noindent We seek values for $q_t$ given an $(r_{x_t},r_{y_t})$.
To do so, we define the difference between a desired $(r_x,r_y)$ tip pose versus that generated by a given $q$ by subtracting the right-hand sides of eqns. (\ref{eq:x}) and (\ref{eq:y}), as in, 

\begin{align}
    d_x(q) & = r_{x_t} - L_1 \sin(q_1)/q_1 \hdots (1-\cos(q_2)) \\
    d_y(q) & = r_{y_t} - L_1(1-\cos(q_1)/q_1 \hdots (1-\cos(q_2))
\end{align}

\noindent and then solve for $q_t$ such that

\begin{equation}\label{eqn:d_equals_0}
    d(q_t) = \begin{bmatrix}
    d_x(q_t) \\ d_y(q_t)
\end{bmatrix} = \begin{bmatrix}
    0 \\ 0
\end{bmatrix}
\end{equation}

\noindent We use MATLAB's built-in numerical solver for this purpose, e.g., {\texttt{qt = fsolve(d, qtminus1)}}, where $q_{t-1}$ is the pose for the previous timestep in the trajectory, used as the initial guess for the solver.
This converts a set of kinematics calculations into a motion planning problem.

Many solutions of $q$ exist for a time series trajectory of $r$, and so we must choose between them.
In this manuscript, we prompt based on the concavity of the robot (``curled counterclockwise or curled clockwise'') as desired, $q_{(-1)} = [\pi/2, \; \pi/2]^\top$ or $q_{(-1)} = [-\pi/2, \; -\pi/2]^\top$.


\subsection{Algorithm For Real-Time Trajectory Generation}

Finally, we combine our differential flatness calculation with the inverse kinematics calculation to obtain Alg. (\ref{alg:one}).

One of the strengths of our method is that kinematic trajectories are automatically converted into dynamic trajectories.
This occurs because the inverse kinematics outputs $q_t$, but not $\dot q_t$ nor $\ddot q_t$.
We can therefore calculate $\dot q_t$ and $\ddot q_t$ as a function of the timestep in our trajectory, $\Delta t$, by a simple finite difference: 

\vspace{-0.5cm}
\begin{equation}\label{eqn:finitediff}
    \dot q_t = (q_t - q_{t-1})/\Delta t, \quad \quad \ddot q_t = (\dot q_t - \dot q_{t-1})/\Delta t 
\end{equation}

\noindent In practice, our method is sufficiently fast and discretization is sufficiently small (example, $\Delta t = 0.01$ sec) that numerical issues did not arise from this numerical differentiation.


\begin{algorithm}
\caption{Real-Time Trajectory Generation}\label{alg:one}
\KwIn{$r_0, \dots , r_T$}
\KwOut{$x_0, u_0, \dots , x_T, u_T$}
\For{$t = 0$ \KwTo $T$}{
    $q_t \gets \textbf{IK} (r_t)$  \quad \quad \quad \quad \quad \quad \quad \quad (eqn. \ref{eqn:d_equals_0}) \\
    $\dot q_t, \ddot q_t \gets \textbf{FiniteDiff} (q_{0\hdots t})$  \quad \quad \quad (eqn. \ref{eqn:finitediff}) \\
    $\{x_t, u_t\} \gets \textbf{DF}(q_t, \dot{q}_t, \ddot{q}_t)$ \quad \quad \quad (eqn. \ref{eqn:u_diffflat}) \\
}
\end{algorithm}


\section{SIMULATION VALIDATION}
\label{section:simulations}

We validate our approach in simulations of a planar two-segment soft robot manipulator, showing both its validity and speed of computation.

\subsection{Simulation Setup}
To simulate, we form an ODE of $\dot x = f(x,u) = [\dot q^\top \; \ddot q^\top]^\top$, inverting eqn. (\ref{eq:q_dynamics}) as $\ddot q = B^{-1}(u - C \dot q - K q - D \dot q)$, noting again that we take $J_\lambda=I$ for proof-of-concept.
Given an $r_0,\hdots,r_T$, we execute Alg. \ref{alg:one}.
We integrate $f(x,u)$ via MATLAB's {\texttt{ode45}}, taking $u_0,\hdots,u_{T-1}$ as an open-loop input trajectory, starting from the initial condition $x_0$, and obtaining the observed trajectory of states $x(1),\hdots,x(T)$.
From these, we calculate $r(1), \hdots, r(T)$.

We calibrated the constants in eqn. (\ref{eq:q_dynamics}) approximately to a hardware prototype currently under development (Fig. \ref{fig:actualsoftrobot}), with $L_i = 12.8$ cm and $m_i = 0.072$ kg.
Stiffness and damping terms in $K$ and $D$ were tuned by hand. 


The simulated trajectories of $r_0,\hdots r_T$ were designed by fitting a cubic spline through a set of carefully selected control points. 
Points along the spline were corresponding end-effector positions, which we sampled evenly for time discretization.
We chose three representative tip trajectories, A, B, and C, representing various reaching or obstacle avoidance maneuvers.
Each tip trajectory had a $\Delta t = 0.01$ sec and a total time of 10 sec.
These were chosen by inspection.
We note that these trajectories were designed to avoid singularities; future work will include more advanced motion planning and autonomous singularity-avoidance techniques.




\subsection{Simulation Results and Real-Time Performance}


Our results from the procedure above demonstrate that the robot's tip tracks each of the A, B, and C trajectories closely (Fig. \ref{fig:jointangleplots}).
The corresponding motions are physically realistic, e.g., Trajectory A plotted as a time lapse in Fig. \ref{fig:transparency}.
We calculate the open-loop task-space errors as $e(t) = || r_t - r(t)||$, the desired vs. observed tip pose.
The average tip position error for each trajectory is then $e_{avg} = (1/T)\sum_{t} e(t)$, in Table \ref{table1}.
These errors are on the order of micrometers, averaging to $5.9149 \times 10^{-5}$ meters, as expected for a simulation with no noise and no approximations in the trajectory generation method.
The only source of error is the velocity finite difference.

\begin{table}[bht]
\centering
\caption{Trajectory Generation And Open-Loop Control Results}
\label{table1}
\begin{tabular}{|c|c|c|}
\hline
\textbf{Trajectory} & \textbf{Average Tip Error } & \textbf{Real-Time Speedup Factor} \\
\hline
A & $7.6332 \times 10^{-5}$ m & 22.2x \\
\hline
B & $6.0634 \times 10^{-5}$ m & 24.3x \\
\hline
C & $4.0483 \times 10^{-5}$ m & 23.8x \\
\hline
\end{tabular}
\end{table}

During these simulations, we also timed Alg. \ref{alg:one} to evaluate its performance versus real-time.
The average time required for generating one point in the trajectory is $t_{avg} = \frac{1}{T} \sum_{t=1}^{T}(\hat{t}_{t+1}-\hat{t}_t)$ where T is the total number of timesteps, and the $\hat t$ are stopwatch times in between each iteration.
To calculate these, we timed $\hat{t}_t$ at the start (MATLAB's \texttt{tic}) and $\hat{t}_{t+1}$ after completion (MATLAB's \texttt{toc}) of each $t$-th iteration. 
The speedup factor over real-time is then $\Delta t/t_{avg}$, also given in Table \ref{table1}.
For $\Delta t = 0.01$ sec on an Apple Macbook Pro M2 at 3.5GHz and 16GB RAM, a representative $t_{avg}$ was 450 microseconds, leading to a speedup of approximately 23x.

\begin{figure}[!t]
\centering
\includegraphics[width=0.7\linewidth]{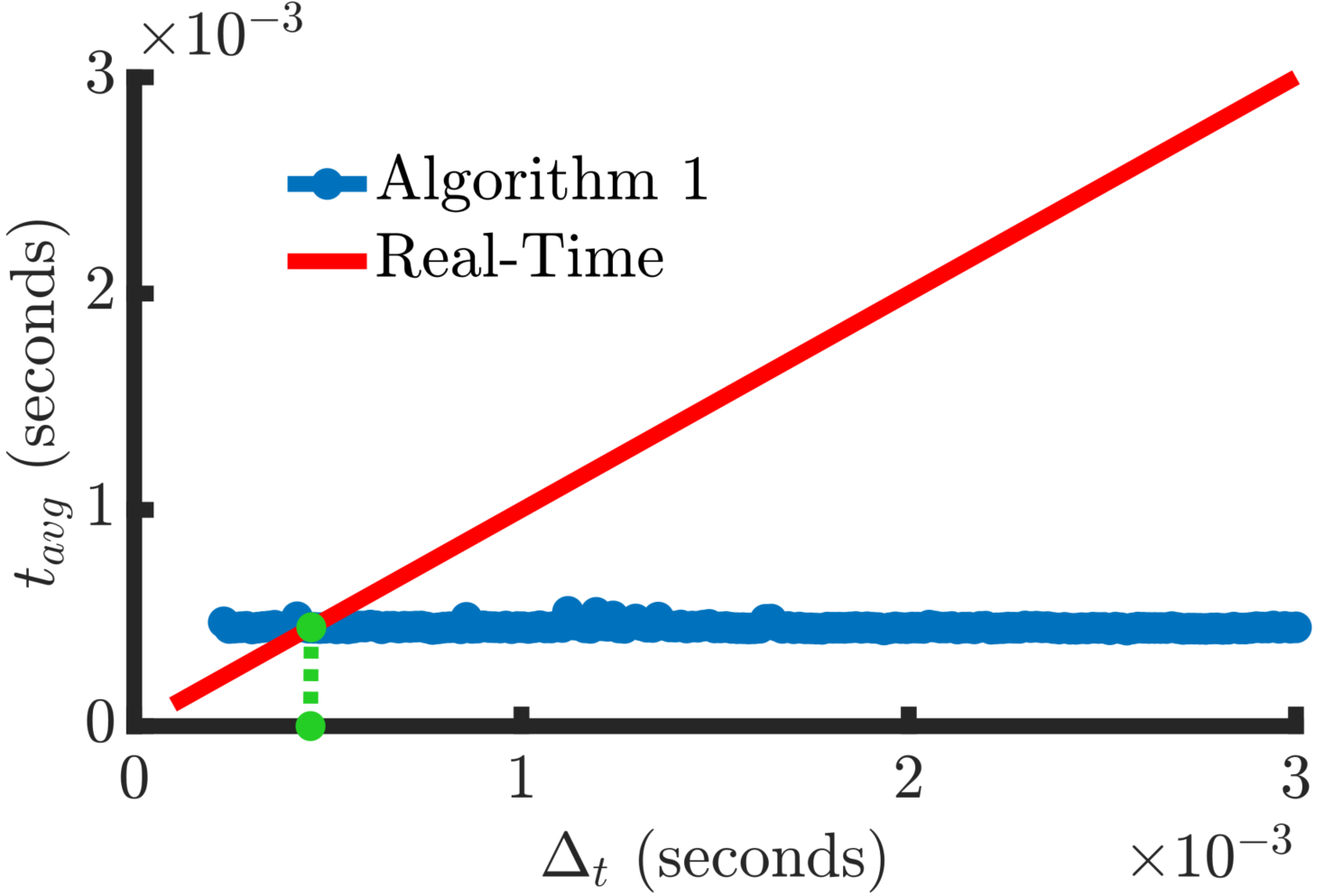}
\caption{Our procedure plans faster than real time whenever the average iteration time for Alg. 1 is less than the planning/control timestep. Our fastest real-time $\Delta t$ is 360 microseconds (green line), which to the authors' knowledge is significantly faster than any other soft robot trajectory generator with dynamic feasibility guarantees.} 
\label{fig:real-time}
\end{figure}

In addition to this real-time computation margin, we are interested in the smallest discretization timestep that still allows us to generate trajectories in real-time.
So, we repeated the simulation procedure for Trajectory A, now with varying $\Delta t$, from $\Delta t = 1\times 10^{-4}$ to $3\times 10^{-3}$, and re-timed $t_{avg}$ at each choice of $\Delta t$.
Fig. \ref{fig:real-time} shows this result, with effectively flat computation time per-iteration regardless of the trajectory timestep.
This is expected, as the computation does not change with $\Delta t$.
More importantly, we can observe that as long as our $\Delta t$ is greater than the crossover point of $\approx 0.00036$ sec, we achieve real-time computation.


Lastly, we confirmed that our results were not an artifact of our setup, or a bug in our calculation.
To do so, we note that the PCC robot's dynamics are asymptotically stable with positive definite $K$ and $D$ matrices \cite{2020DellaModel-Based}, and so if our robot were initialized outside the true $x_0$, it would ideally converge to the goal trajectory $q(t)$ after some amount of time, even without feedback.
We tested this by executing Trajectory A open-loop from three different initial conditions, with the joint angles $q_1(0)$ and $q_2(0)$ not aligned with the reference trajectory (Fig. \ref{fig:diffinitialpositions}).
Our results confirm the open-loop stability proof from \cite{2020DellaModel-Based}, as trajectories converge to the reference within a short time ($\sim$20 timesteps).

\begin{figure}[htb]
\centering
\includegraphics[width=1\linewidth]{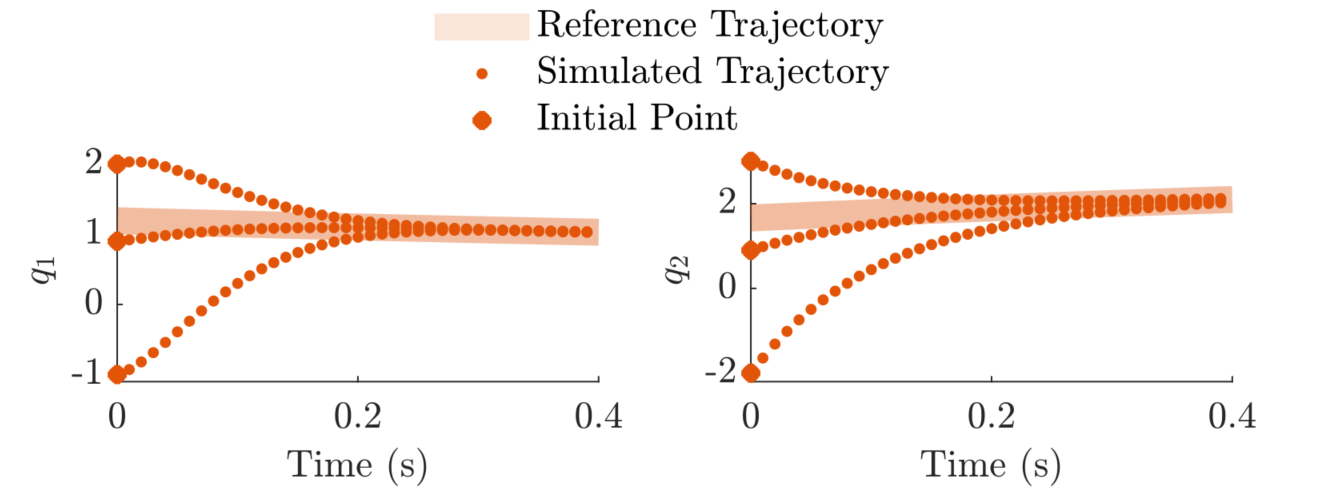}
\caption{We confirm that our results are not an artifact of our test setup by relying on the asymptotic stability of the PCC manipulator dynamics, and initializing the robot at a different set of joint angles, before executing the open-loop $u(0),\hdots,u(T)$ of Trajectory A. The simulation converges to the reference trajectory quickly despite no loop closure, validating that Alg. \ref{alg:one}'s calculated $\{x(t),u(t)\}_{0\hdots T}$ is feasible.} \label{fig:diffinitialpositions}
\end{figure}


\section{DISCUSSION AND CONCLUSION}\label{sec:discusion}

This manuscript introduces a fast trajectory generation procedure for piecewise constant curvature soft robots.
We prove that PCC soft robot dynamics are differentially flat under some assumptions, and demonstrate that our method generates dynamically feasible trajectories much faster than real time.
To our knowledge, this is the first work to apply differential flatness to soft robots, and is a significant increase in computational performance of motion planning for soft manipulators that respects dynamic feasibility.

\begin{figure}[thb]
\centering
\includegraphics[width=0.8\linewidth]{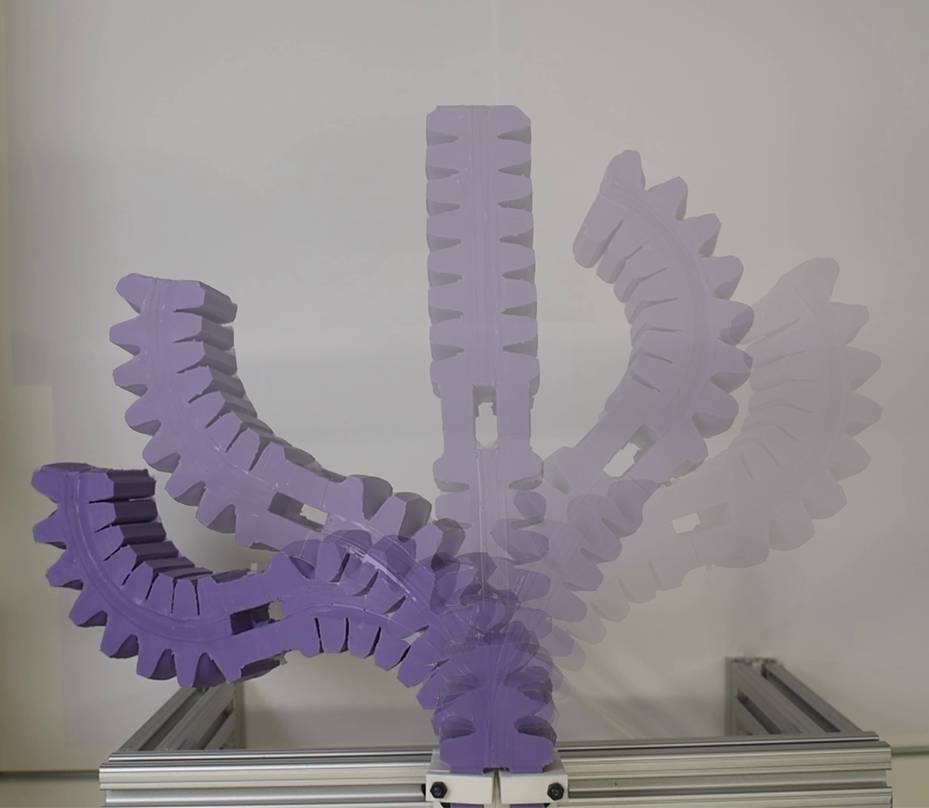}
\caption{Two-segment fully-actuated pneumatic soft robot manipulator prototype used for calibrating the simulation. Future work will implement our trajectory generation method in this hardware.} 
\vspace{-0.4cm}
\label{fig:actualsoftrobot}
\end{figure}

These results are attractive for many-segment three-dimensional soft robots, particularly with respect to feedback control.
One use for online motion planning is such a controller: replanning a motion at each timestep in closed-loop.
Typical prior work has used hardware control periods of 50-100 milliseconds \cite{Wertz2022Trajectory,patterson_robust_2022}, so this planning method has a factor of $\sim$100x margin of time for computation.
It is therefore likely that more complicated robots could plan motions in feedback in real time, overcoming the limitations of more complicated models, e.g. Cosserat rod models.
More generally, such a ``time margin'' could be viewed as a type of robustness in formal logic \cite{belta_formal_2019}, where a new control input must be recalculated within a period in order to maintain a safety specification.
Future work will examine these possibilities for soft robots.


However, this proof-of-concept has limitations with respect to its simplicity.
While the assumptions inherent in the dynamics model are straightforward to satisfy for slender soft manipulators, our method currently does not account for input constraints, state constraints or environmental obstacles, or infeasible end effector positions.
Future work will address these, including motion planning for the tip trajectory itself.
And, our method only treats the task space goal as $r \in \mathbb{R}^2$ rather than a pose in $SE(3)$, which could be rectified by more advanced inverse kinematics.
Finally, future work will address singularities, possibly by adopting a singularity-free dynamics for PCC manipulators \cite{della2020imporved}.

Lastly, the proposed method is tested in simulation only.
The ``sim2real'' gap in soft robotics can be significant for many modeling approaches \cite{choi_dismech_2024}.
However, the PCC dynamics used here are well-validated against hardware \cite{Della2023Model}, and a prototype exists that meets the actuation assumptions for differential flatness.
We are therefore cautiously optimistic about the validity of our method.
Ongoing work will apply our approach to this hardware prototype, under development, in Fig. \ref{fig:actualsoftrobot}.

\bibliographystyle{IEEEtran}
\bibliography{references}
\end{document}